\documentclass{article}
\usepackage{graphicx} 
\usepackage{hyperref} 
\usepackage{amsmath, amsthm, amssymb}
\usepackage{tikz}
\usepackage{tcolorbox}
\tcbuselibrary{skins}
\usepackage{dashrule} 
\usepackage{xcolor}
\usepackage{wrapfig}
\usepackage{fullpage}

\makeatletter
\newcommand*{\rom}[1]{\expandafter\@slowromancap\romannumeral #1@}
\makeatother

\usepackage[ruled, lined, linesnumbered, longend]{algorithm2e}

\newtheorem{theorem}{Theorem}

\newtheorem{lemma}{Lemma}

\newtheorem{definition}{Definition}

\title{A completely uniform transformer for parity}

\author{ Alexander Kozachinskiy$^{1}$ \and Tomasz Steifer$^{2,3}$}
\date{%
    $^1$Centro Nacional de Inteligencia Artificial, Chile\\%
  $^2$Instituto de Ingeniería Matemática y Computacional, Universidad Católica de Chile\\
    $^3$Institute of Fundamental Technological Research, Polish Academy of Sciences\\%
}

\begin{document}
\maketitle
\begin{abstract}
    We construct a 3-layer constant-dimension transformer, recognizing the parity language, where neither parameter matrices nor the positional encoding depend on the input length. This improves upon a construction of Chiang and Cholak who use a positional encoding, depending on the input length (but their construction has 2 layers).
\end{abstract}

\section{Introduction}
One of the ways to do mathematical analysis 
of the capabilities and limitations of the transformer architecture~\cite{vaswani2017attention}  is to study formal languages, recognizable by them. Namely, for a given formal language $L$, we study, if there exists a choice of parameters in the transformer architecture, for which words from $L$ are accepted and words not from $L$ are rejected by the resulting transformer. A seminal work of Hahn~\cite{hahn2020theoretical} performed such analysis for a number of formal languages, including the \emph{parity language}, consisting of binary words with even number of 1s. 
Hahn have shown that transformers, recognizing this language, must have low confidence,  partially explaining  an empirically observed struggle of transformers in learning this language~\cite{bhattamishra2020ability, deletang2022neural}.

However, this does not exclude the possibility of existence of a theoretical solution for this language. And indeed, as was shown by Chiang and Cholak~\cite{Chiang2022OvercomingAT}, there exists a 2-layer transformer with constant embedding dimension, recognizing the parity language. Their construction is \emph{uniform} in a sense that parameter matrices in it do no depend on the input length. However, there is one aspect of their construction which is not completely uniform -- the \emph{positional encoding}. At position $i$, they use $i/n$, where $n$ is the input length. This means that their positional encoding has to be reset each time we want to use their transformer for a larger input length.

In this paper, we get rid of this disadvantage of their construction by giving a \emph{completely uniform} 3-layer transformer, recognizing parity. That is, in our construction, neither parameter matrices nor positional encoding depend on $n$. We do not use neither positional masks nor layer norm.

It remains open if the parity language can be recognized by a 1-layer transformer with constant embedding dimension, even non-uniformly. Current lower bound methods against 1-layer transformers~\cite{bhattamishra2024separations,peng2024limitations,DBLP:conf/nips/SanfordHT23} do not seem to work against parity.

\section{Preliminaries}
An attention layer of dimension $d$ is a length-preserving function $A\colon(\mathbb{R}^d)^*\to(\mathbb{R}^d)^*$,  given by three matrices $K, Q, O\in\mathbb{R}^{d\times d}$ and a neural network $\mathcal{N}\colon\mathbb{R}^d\to\mathbb{R}^d$ with ReLU activation. On input $(f_1, \ldots, f_n)\in(\mathbb{R}^d)^n$, the output $A(f_1, \ldots, f_n)$ is computed as follows. First, we define:
\[a_i = \frac{\sum\limits_{j = 1}^n e^{\langle K f_j, Q f_i\rangle} f_j}{\sum\limits_{j = 1}^n e^{\langle K f_j, Q f_i\rangle}}, \qquad i = 1, \ldots, n,\]
and set:
\[A(x_1 \ldots x_n)_i = \mathcal{N}(f_i + Oa_i), \qquad i = 1, \ldots, n. \]
 Applying $\mathcal{N}$ allows us to do arbitrary piece-wise linear transformation of $\mathbb{R}^d$.
\begin{definition}
    We say that a language $L\subseteq\{0, 1\}^*$ is recognizable by a completely uniform transformer with $C$ layers if there exists $d\in\mathbb{N}$, $C$ attention layers $A_1, \ldots, A_C$ of dimension $d$, and a letter embedding  $\ell\colon\{0, 1\}\to\mathbb{R}^d$, and a position encoding $  p\colon\mathbb{N}\to\mathbb{R}^d$ such that for any $n\in\mathbb{N}$, and for any $x = x_1 x_2 \ldots x_n\in\{0, 1\}^n$, the following holds. Define:
    \[f_1 = \ell(x_1) + p(1), \ldots, f_n = \ell(x_n) + p(n),\]
    and set
    \[(g_1, \ldots, g_n) = A_C\circ \ldots \circ A_1(f_1\ldots f_n).\]
    Then the following must hold. If
    $x\in L$, we have $g_1^1 > 0$, and if $x\notin L$, we have $g_1^1 < 0$. 
\end{definition}
Note that in our definition, the positional encoding $p$ does not take $n$, the input length, as an input.
\section{Construction}
\begin{lemma}
\label{lem_f(n)}
    For any function $f\colon\mathbb{N}\to\mathbb{R}$, there exists a completely uniform transformer that, for any input length $n$, computes $f(n)$ in every position in one layer.
\end{lemma}
\begin{proof}
    We need a positional encoding $p\colon\mathbb{N}\to\mathbb{R}$ such that:
    \[\frac{p(1) + \ldots + p(n)}{n} = f(n),\]
    for every $n$ (this average is computable via softmax with uniform weights),
    which is achievable by setting $p(1) = f(1)$ and $p(i) = if(i) - (i - 1) f(i - 1)$ for $i\ge 2$,
\end{proof}

\begin{theorem}
    There exists a 3-layer completely uniform transformer, recognizing the parity language.
\end{theorem}
\begin{proof}

Assume that on input we get $x = x_1\ldots x_n\in\{0, 1\}^n$. Denote $\Sigma = x_1 + \ldots + x_n$. 
Let us first give a construction, assuming that $\Sigma \ge 1$, that is, that not all input bits are 0. We explain how to get rid of this assumption in the end of the proof.

The plan of the proof is as follows. At the first two layers, we need to compute a sequence of numbers
\[a_1, \,\, a_2, \ldots,\,\, a_n,\]
with the property that $a_\Sigma$ is strictly larger than any other number in the sequence. Assuming we have done this, at the last layer we can compute the following:
\begin{equation}
    \label{eq_limit}
  \frac{\sum\limits_{i = 1}^n \exp\{a_i f(n)\} \cdot (-1)^i}{\sum\limits_{i = 1}^n \exp\{a_i f(n)\}}.
\end{equation}
using a positional encoding $i\mapsto (-1)^i$ and Lemma \ref{lem_f(n)} for a function $f\colon\mathbb{N}\to\{0, 1\}$ of our  choice. Our goal for \eqref{eq_limit} is to encode parity, i.e., \eqref{eq_limit} must be positive if $\Sigma$ is even and negative otherwise.
Indeed, 
the limit of \eqref{eq_limit} as $f(n)\to +\infty$ is $(-1)^\Sigma$, because $a_\Sigma > a_i$ for $i\neq \Sigma$. For any fixed $n$, there are finitely many inputs of length $n$.  Hence, we can take $f(n)$ large enough so that, for any input $x$ of length $n$, the difference between \eqref{eq_limit} for this $x$ and $(-1)^\Sigma$ is at most $1/3$.

It remains to calculate a sequence $a_1, \ldots, a_n$ with this property at the first two layers.  By Lemma \ref{lem_f(n)}, we compute $\ln n $ at every position in the first layer. Then, for an absolute constant $\delta\in\mathbb{R}$, to be specified later, we set $\alpha = e^\delta$ and use the second layer to compute the following expression at every position:
\begin{equation}
    \label{eq_gamma}
    \gamma = \frac{\sum\limits_{i = 1}^n e^{(-\ln(n) + \delta)\cdot(1 - x_i)}\cdot x_i}{\sum\limits_{i = 1}^n e^{(-\ln(n) + \delta)\cdot(1 - x_i)}} = \frac{\Sigma}{\Sigma + (n - \Sigma)\cdot \frac{\alpha}{n}} = \frac{1}{1 + \frac{\alpha}{\Sigma} - \frac{\alpha}{n}}.
\end{equation}

It now suffices to prove the following lemma.
\begin{lemma}
\label{super_lemma}
    There exists $\alpha\in(0, 1)$ such that for all $n$ and $\Sigma\in\{1, \ldots, n\}$,
    the maximum of the expression:
    \[a_i = -\left|\gamma - 1 + \left(\frac{\alpha}{i} - \frac{\alpha}{n}\right) - \left(\frac{\alpha^2}{i^2} +\frac{\alpha^2}{n^2}\right)\right|\]
    over $i\in\{1, \ldots, n\}$
    is attained uniquely at $i = \Sigma$, where $\gamma$ is defined by \eqref{eq_gamma}.
\end{lemma}
Indeed, once we have computed $\gamma$, we can express $a_i$  by computing $\frac{\alpha}{n} - \frac{\alpha^2}{n^2}$ in every position by Lemma \ref{lem_f(n)}, and using a positional encoding $i\mapsto \frac{\alpha}{i} - \frac{\alpha^2}{i^2}$.  We can then calculate the absolute value using a constant-size ReLU network.
\begin{proof}[Proof of Lemma \ref{super_lemma}]
Observe that:
\[\left|\frac{1}{1 + z} - (1 - z + z^2)\right| = \frac{z^3}{1 + z} \le z^3\]
for $z \ge 0$. Applying this to $z = \frac{\alpha}{\Sigma} - \frac{\alpha}{n}$, we get:
\[\gamma = 1 - \left(\frac{\alpha}{\Sigma} - \frac{\alpha}{n}\right) + \left(\frac{\alpha}{\Sigma} - \frac{\alpha}{n}\right)^2 + \rho,  \]
for some $\rho$ with $|\rho|\le \frac{\alpha^3}{\Sigma^3}$. We now can write:
\[a_i = -\left|\left(\frac{\alpha}{i}-\frac{\alpha}{\Sigma}\right) + \left(\frac{\alpha^2}{\Sigma^2}-\frac{\alpha^2}{i^2}\right) - \frac{2\alpha^2}{\Sigma n} + \rho\right|\]
Denoting $b_i = \left(\frac{\alpha}{i}-\frac{\alpha}{\Sigma}\right)$, $c_i = \left(\frac{\alpha^2}{\Sigma^2}-\frac{\alpha^2}{i^2}\right)$ and $\lambda = - \frac{2\alpha^2}{\Sigma n}$, we get:
\[a_i = - |b_i + c_i + \lambda + \rho|\]
For $i = \Sigma$, we get $a_\Sigma = -|\lambda + \rho|$.  To prove that the maximum of $a_i$ is attained at $i = \Sigma$, we show that:
\[(1/10)\cdot |b_i| >  | c_i|, \qquad (1/10)\cdot|b_i|>  |\lambda|,\qquad (1/10)\cdot|b_i|>  |\rho|,\qquad \text{ for } i \neq \Sigma,\]
which implies that:
\begin{align*}
    -a_i = |b_i + c_i + \lambda + \rho| \ge  |b_i| - |c_i| - |\lambda| - |\rho|\ge (7/10) |b_i| > |\lambda| + |\rho| \ge -a_\Sigma.
\end{align*}
Firstly, for $i\neq \Sigma$, observe:
\begin{equation}
\label{eq_c}
    \frac{|b_i|}{|c_i|} = \frac{i \Sigma}{\alpha(\Sigma + i)}\ge \frac{1}{2\alpha},
\end{equation}
where the last inequality is due to the fact that $i, \Sigma \ge 1$. Next, for $i\neq \Sigma$, observe:
\begin{equation}
\label{eq_lambda}
    \frac{|b_i|}{|\lambda|} = \frac{\frac{\alpha|i - \Sigma|}{i \cdot \Sigma}}{\frac{2\alpha^2}{\Sigma n}}\ge \frac{1}{2\alpha}.
\end{equation}
Finally, for $i\neq \Sigma$, write:
\[\frac{|b_i|}{\rho}\ge \frac{\frac{\alpha|i - \Sigma|}{i \cdot \Sigma}}{\frac{\alpha^3}{\Sigma^3}}\ge \frac{1}{\alpha^2}\cdot \frac{|i - \Sigma| \cdot \Sigma}{i}\]
If $\Sigma \ge i/2$, since $|i - \Sigma|\ge 1$, we obtain $\frac{|b_i|}{\rho}\ge \frac{1}{2\alpha^2}$. If $\Sigma \le i/2$, we obtain $|i - \Sigma|\ge i/2$, implying the same bound
\begin{equation}
\label{eq_rho}
    \frac{|b_i|}{\rho}\ge \frac{1}{2\alpha^2},
\end{equation}
since $\Sigma \ge 1$. Taking $\alpha = 1/100$, we obtain that the right-hand sides of the  inequalities in \eqref{eq_c}, \eqref{eq_lambda}, \eqref{eq_rho} are all larger than  10, as required.
\end{proof}

Finally, we explain how to get rid of the assumption $\Sigma \ge 1$. Let us denote the value of the expression \eqref{eq_limit} by $\theta$. We have that $\theta$  is positive for even $\Sigma\ge 1$, and $\theta$ is  negative for odd $\Sigma \ge 1$. Now, at the first layer, we can compute the quantity 
$1/(2n) - (x_1 + \ldots + x_n)/n$ by taking the arithmetic mean of the input bits using sotftmax with uniform weights, and computing $1/(2n)$ by Lemma \ref{lem_f(n)}. Observe this quantity is positive for $\Sigma = 0$ and negative otherwise. It remains to output $\max\{\theta, 1/(2n) - (x_1 + \ldots + x_n)/n\}$ in the third layer.

\end{proof}

\end{document}